\documentclass[10pt,twocolumn,letterpaper]{article}

\usepackage{cvpr}              % To produce the CAMERA-READY version
\usepackage{graphicx}
\usepackage{amssymb}
\usepackage{booktabs}
\usepackage{amssymb,amsmath,amsthm}

\newtheorem{theorem}{Theorem}

\usepackage{bm}
\usepackage{color}
\usepackage{xcolor}
\usepackage{url}
\usepackage{tabularx}
\usepackage{longtable}
\usepackage{arydshln}
\usepackage{comment}
\usepackage{listings}
\usepackage[pagebackref,breaklinks,colorlinks]{hyperref}

\usepackage[capitalize]{cleveref}
\crefname{section}{Sec.}{Secs.}
\Crefname{section}{Section}{Sections}
\Crefname{table}{Table}{Tables}
\crefname{table}{Tab.}{Tabs.}

\renewcommand{\paragraph}[1]{\hspace{-3mm}{\bf #1} \hspace{1mm}}
\newcommand{\ryoshiha}[1]{\textcolor{black}{#1}}

\def\cvprPaperID{1336} % *** Enter the CVPR Paper ID here
\def\confName{CVPR}
\def\confYear{2023}

\begin{document}

%%%%%%%%% TITLE - PLEASE UPDATE
\title{Ladder Siamese Network: a Method and Insights \\for Multi-level Self-Supervised Learning}
%\title{Ladder Siamese Network: Versatile Representation Hierarchy \\ via Multi-level Self-Supervised Learning}

\author{Ryota Yoshihashi \and
Shuhei Nishimura \and
Dai Yonebayashi \and
Yuya Otsuka \and
Tomohiro Tanaka \and
Takashi Miyazaki \\
\hspace{-47mm}Yahoo Japan Corporation\\
\hspace{-45mm}{\tt\small ryoshiha@yahoo-corp.jp}
% For a paper whose authors are all at the same institution,
% omit the following lines up until the closing ``}''.
% Additional authors and addresses can be added with ``\and'',
% just like the second author.
% To save space, use either the email address or home page, not both
%\and
%Second Author\\
%Institution2\\
%First line of institution2 address\\
%{\tt\small secondauthor@i2.org}
}
\maketitle

%%%%%%%%% ABSTRACT
\begin{abstract}
  Siamese-network-based self-supervised learning (SSL) suffers from slow convergence
  and instability in training.
  To alleviate this, we propose a framework to exploit intermediate self-supervisions in each stage of deep nets, called the {\it Ladder Siamese Network}.
   Our self-supervised losses encourage the intermediate layers
   to be consistent with different data augmentations to single samples, which facilitates training progress and enhances the discriminative ability of the intermediate layers themselves.
   While some existing work has already utilized multi-level self supervisions in SSL, ours is different in that 1) we reveal its usefulness with non-contrastive Siamese frameworks in both theoretical and empirical viewpoints, and 2) ours improves image-level classification, instance-level detection, and pixel-level segmentation simultaneously.
   Experiments show that the proposed framework can improve BYOL baselines by 1.0\% points in ImageNet linear classification, 1.2\% points in COCO detection,
   and 3.1\% points in PASCAL VOC segmentation.
   In comparison with the state-of-the-art methods, our Ladder-based model achieves competitive and balanced performances in all tested benchmarks without causing large degradation in one. 
\end{abstract}

\begin{figure}[t]
    \centering
    \includegraphics[width=0.93\linewidth]{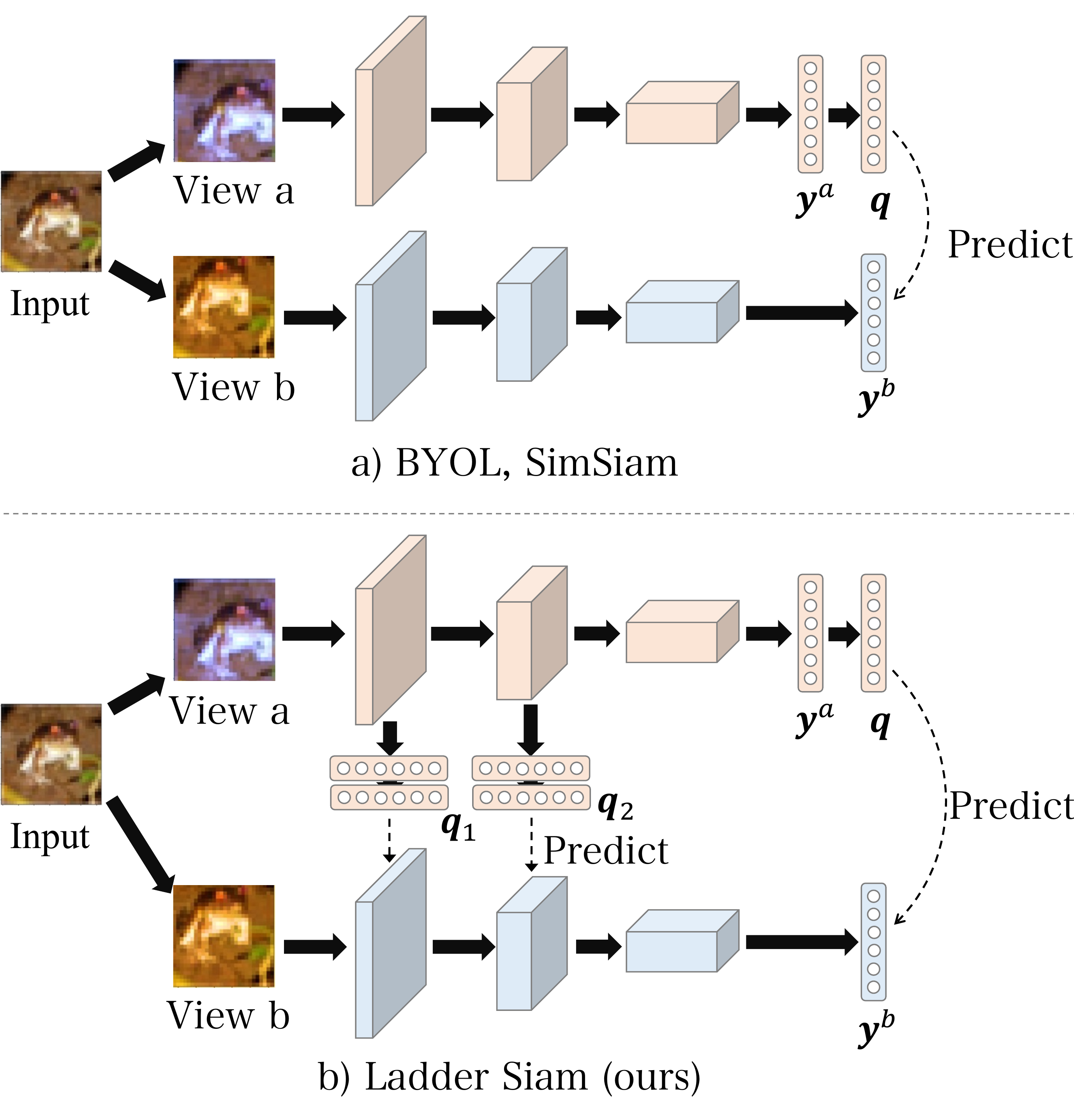}\vspace{-2mm}
    \caption{Illustration of a) existing major Siamese SSL methods and b) our Ladder Siam framework. We exploit intermediate-layer self-supervisions, to stabilize the training process and to enhance intermediate-layer reusability in downstream tasks. }\vspace{-4mm}
    \label{fig:overview}
\end{figure}

%%%%%%%%% BODY TEXT
\vspace{-4mm}
\section{Introduction}\vspace{-1mm}
\label{sec:intro}
Conventional deep neural networks are notoriously {\it label-hungry},
requiring massive human-annotated training data to demonstrate their full performance \cite{sun2017revisiting}.
Self-supervised learning (SSL) \cite{chen2020simple} %総説か代表的な論文をここにいれる
is a promising approach
to reduce this annotation dependency of deep nets, and to make machine learning more autonomous toward enabling more human-like learning mechanisms.

Among various SSL methods, one of the promising approaches is cross-view learning with Siamese networks \cite{chen2020simple,he2020momentum,grill2020bootstrap}.
This approach is further divided into two types: contrastive and non-contrastive methods.
Contrastive methods include the pioneer work SimCLR \cite{chen2020simple}, which
makes pairs of representations from the same instances similar and
pairs from different instances dissimilar.
While its training using positive and negative pairs is intuitive,
the selection of negative pairs may affect performance and it is more sensitive to the batch sizes \cite{arora2019theoretical,nozawa2021understanding}.
Non-contrastive methods, including the commonly-used BYOL \cite{grill2020bootstrap} and SimSiam \cite{chen2021exploring}, 
eliminate the necessity for negative pairs by defining
loss functions that only depend on positive pairs, which empirically improves learned representations.

A common problem in the Siamese SSL is slow convergence and instability during training.
In non-contrastive methods, existence of trivial solutions exacerbates the problem.
The non-contrastive training may fall into the trivial solutions
that map every input signal to a constant vector, which is called a {\it collapse}.
While existing studies observed that the collapse can be avoidable
in carefully designed training frameworks \cite{chen2020simple,hua2021feature,tian2021understanding},
it is still a problem in certain settings depending on model choices and training-dataset sizes \cite{li2022understanding}.

To alleviate this difficulty of the Siamese SSL,
we propose a training framework to exploit multi-level
self-supervision.
The multi-level self-supervision encourages each stage of representations in the hierarchical networks to be consistent
to different data augmentations within the same images. 
\ryoshiha{This is expected to 1) enhance the progress of training
of the earlier stages by directly exposing them to the loss,
and
%2) stabilize trajectories of the parameter vectors during training by giving more constraints to the model, which would be useful to prevent training failure, 
2) as a side effect, improve the discriminative ability of the middle layers themselves.
One might be concerned about such an aggressive loss addition to all levels in non-contrastive frameworks known to cause the collapse.
However, from theoretical viewpoints, we argue that multiple losses added by multi-level supervisions do not 
increase the range of trivial solutions' existence, which might be counterintuitive.
\ryoshiha{Figure \ref{fig:overview} shows the overview of the architecture.}}
We name our framework {\it Ladder Siamese Network} 
after an autoencoder-based un- and semi-supervised learning method \cite{rasmus2015semi} with
multi-level losses in the pre-SSL era.

In the context of cross-view SSL, 
multi-level self supervision has been partly examined
as a component of task-specialized SSL methods.
For example, DetCo \cite{xie2021detco} is a contrastive method to exploit
multi-level self-supervision and local-patch-based training for detection.
In contrast, we explore multi-level self-supervision as a pretraining method
for generic-purpose backbones.
We found that in combination with non-contrastive losses, 
multi-level self-supervisions (MLS) can improve classification, detection, and segmentation performance
simultaneously, while DetCo had to sacrifice
classification accuracy to improve detection performance with the multi-level self-supervision.
%We also reveal multi-level supervision's effectiveness from theoretical view points; we argue that multiple losses added by multi-level supervisions do not 
%increase the risk of a collapse by increasing the range of trivial solutions' existence, which might be counterintuitive.
In addition, we find that dense self-supervised losses \cite{wang2021dense,xie2021propagate}, which were to boost detection and segmentation performances at the cost of classification, can be exploited with much less classification degradation in the Ladder Siam framework \ryoshiha{than in conventional top-level-loss-only settings}.

Our contributions are summarized as follows. % 何か書く、考え中
First, we propose a non-contrastive multi-level SSL framework, Ladder Siam.
Second, in experiments, we show that Ladder Siam is useful to build a representation hierarchy that maintains competitive performances to the state-of-the-art methods in classification, detection, and segmentation simultaneously.
Third, we extensively analyze the role of multi-level self-supervision in training from theoretical and experimental view points and show how the intermediate losses work to compliment the top-level supervision.
Code and pretrained weights will be released upon acceptance.

%-------------------------------------------------------------------------
\section{Related Work}
\label{sec:related}
\vspace{-1mm}\subsection{Siamese self-supervised learning}\vspace{-2mm}
Siamese SSL was derived from the line of
instance-discrimination-based SSL \cite{dosovitskiy2014discriminative, wu2018unsupervised}.
Instance discrimination is a proxy task to
identify differently data-augmented versions of single images, which is conceptually 
simpler than the previous proxy tasks  \cite{noroozi2016unsupervised,gidaris2018unsupervised,pathak2016context}.
The early instance-discrimination method is 
parametric \cite{dosovitskiy2014discriminative} to perform a classification where each training instance is a class. Afterwards, a memory bank that stores per-instance weights \cite{wu2018unsupervised} was introduced to improve scalability by being non-parametric.
Finally, the memory bank was replaced by the Siamese-net-style
dual encoders \cite{bromley1993signature, koch2015siamese} that compute per-instance weights from the second view of the input on-the-fly \cite{chen2020simple}.

Many of the Siamese SSL methods after the pioneer SimCLR \cite{chen2020simple}
have similar overall architecture \cite{chen2021exploring} and
explore various loss functions, prediction modules, and network update rules.
For example, MoCo \cite{he2020momentum,chen2020improved,chen2021empirical} introduced momentum encoders, which are slowly updated during training to increase targets' time-consistency. SwAV \cite{caron2020unsupervised} exploits online clustering and predicts cluster assignments, rather than representation vectors themselves, to improve stability. BYOL \cite{grill2020bootstrap} adopts the asymmetric predictor,
which is on only the one side of the Siamese net, and incorporated a non-contrastive loss that does not rely on negative samples.
Further sophistication of prediction modules \cite{huang2022learning,pang2022unsupervised}, loss-function design \cite{zbontar2021barlow,bardes2022vicreg,tao2022exploring}, and data augmentation \cite{tian2020makes,miyai2022rethinking,hayase2022downstream} within the Siamese architecture is ongoing.
However, we argue that architectural changes, such as the addition of intermediate losses, have not been deeply investigated. 

Nevertheless, we are aware of a few studies that adopted multi-level self-supervisions (MLS) in Siamese SSL. 
DetCo \cite{xie2021detco} used MLS in combination with local patch-wise contrastive learning.
CsMl \cite{xu2022seed} combined MLS with nearest-neighbor-based positive-pair augmentation.
HCCL \cite{chen2021hierarchical} incorporated MLS in its deep projection heads rather than in the backbone.
Hierarchical Augmentation Invariance \cite{zhang2022rethinking}
assigned specific data augmentation types for each level to learn invariance against them.
Remarkably, all of the work presented MLS in bundles to boost the system performance after combined them with other ideas.
We instead focus on the analyses of vanilla MLS and show that even a straightforward implementation based on BYOL can outperform the preceding MLS methods.
The ideas based on MLS have also been examined in other domains, such as video \cite{yang2020hierarchical} and medical images \cite{kaku2021intermediate}.

A number of Siamese SSL methods incorporate region- or pixel-wise learning, which is useful to improve locality awareness and spatial granularity of representations.
Region-based methods often use an extra region-proposal module.
For example, DetCon \cite{henaff2021efficient} utilizes multiscale combinatorial grouping \cite{arbelaez2014multiscale}, SoCo \cite{wei2021aligning} and UniVIP \cite{li2022univip} utilize selective search \cite{uijlings2013selective},
and CYBORGS \cite{wang2022cyborgs} and Odin \cite{henaff2022object} utilize region grouping by k-means \cite{lloyd1982least}
to define region-to-region losses.
While they are effective especially in object detection, 
the usage of a handcrafted region-proposal may cause implementation complexity and loss of generality, for example, when applying to non-object image datasets such as scenes or textures.
In contrast, we explore a method that does not rely on extra modules yet can improve
detection and segmentation.

Pixel-wise methods eliminate global pooling and aim to define dense supervision.
For example, DenseCL \cite{wang2021dense} exploits the dense correspondence between feature maps. PixPro \cite{xie2021propagate} utilizes coordinate-based alignment by tracing the cropping-based data augmentation. VICRegL \cite{bardes2022vicregl} boosts the dense SSL with VIC regularization \cite{bardes2022vicreg}, and there are more studies along this line to improve matching strategy \cite{li2021dense,wang2022exploring}.
DenseSiam \cite{zhang2022densel} incorporates both dense and region-based
learning.
\ryoshiha{We incorporate DenseCL, the simplest one in our Ladder framework with a non-contrastive modification.}

\vspace{-1mm}\subsection{Intermediate-layer supervision}\vspace{-2mm}\
Intermediate-layer supervision has been examined in various areas since the advent of deep neural networks to alleviate their training difficulty.
Our direct source of inspiration is LadderNet and its variants \cite{valpola2015neural,rasmus2015semi,yoshihashi2019classification}
that perform autoencoder-based denoising \cite{vincent2010stacked} of intermediate representations as additional supervisions.
Deeply supervised nets \cite{lee2015deeply} exploits
classification losses on intermediate layers, which has been incorporated
in more supervised methods \cite{szegedy2015going,zhao2017pyramid}.
Deep contrastive supervision \cite{zhang2022deep} is a supervised learning method that exploits SSL-like contrastive losses on intermediate layers as regularizers. While it is related to our method in terms of the intermediate-loss usages, we investigates purely self-supervised settings. 
Knowledge distillation is another area where intermediate-layer supervision is common to give student models more hints to mimic teacher models \cite{romero2014fitnets, yang2021hierarchical}.
However, they use teachers trained with supervised learning and are largely different to our SSL setting, where the dual encoders are simultaneously updated.
In the broadest sense, methods that encourages reuse of intermediate layers by lateral connections \cite{ronneberger2015u,lin2017feature,misra2016cross,kawakami2019cross} could be seen as forms of intermediate-layer supervisions.
%-------------------------------------------------------------------------
\section{Method}
\label{sec:method}

\subsection{Preparation: Siamese SSL}
First, we briefly review the Siamese SSL framework \cite{grill2020bootstrap,chen2021exploring} as a background and introduce notations.
Given an input $\bm{x}$, a deep network with $N$ stages that maps
$\bm{x}$ to the output $\bm{y}$ generally can be written as
\begin{eqnarray}\label{eqn:deepnet} \vspace{-2mm}
    \bm{y} &=& \bm{f}_N(\bm{z}_{N-1}) \nonumber \\
    \bm{z}_i &=&  \bm{f}_i(\bm{z}_{i-1}) \quad\qquad (i = 1, 2, ..., N-1) \\
    \bm{z}_0 &=& \bm{x},  \nonumber \vspace{-2mm}
\end{eqnarray}
where $\bm{f}_i$ denotes the $i$-th stage of the network and
$\bm{z}_i$ denotes the intermediate representations produced by $\bm{f}_i$.
Here, {\it stages} mean certain groups of layers in networks (i.g., conv1, res2, res3, ... in ResNets \cite{he2016deep}),
typically grouped by their resolutions and divided by downsampling layers.
The composite function
\begin{eqnarray}\label{eqn:composite} 
\bm{f}(\bm{x}) = (\bm{f}_N \circ \bm{f}_{N-1} \circ ... \circ \bm{f}_{1})(\bm{x})
\end{eqnarray}
denotes the whole network as a single function.

In supervised learning, the loss function that compares the outputs $\bm{y}$
and the annotated labels drives the training forward.
However, in SSL, we need an alternative to the label.
Here, Siamese frameworks exploit two {\it views} of single instances, which are
two versions of input images differently augmented by random transformation.
Given an input $\bm{x}$, using its two views $\bm{x}^a$, $\bm{x}^b$ and 
their corresponding outputs $\bm{y}^a = \bm{f}(\bm{x}^a)$, $\bm{y}^b = \hat{\bm{f}}(\bm{x}^b)$, a self-supervised
loss function is defined as $L(\bm{y}^a, \bm{y}^b)$.
The network for the second view $\hat{\bm{f}}$ may be identical to $\bm{f}$
\cite{chen2020simple}, or the slowly-updated version of $\bm{f}$ with momentum \cite{grill2020bootstrap}.

An example of the concrete form of $L(\bm{y}^a, \bm{y}^b)$ is the mean-square errors
(MSE) with a predictor, introduced by BYOL \cite{grill2020bootstrap}, which is denoted by
\begin{eqnarray}\label{eqn:loss}
L_{\text{BYOL}}(\bm{y}^a, \bm{y}^b) &=& |\bm{q}(\bm{y}^a) - \bm{y}^b|_2^2 ,
%L_{\text{BYOL}}(\bm{y}^a, \bm{y}^b) &=& -<\bm{q}(\bm{y}^a) , \bm{y}^b> ,
\end{eqnarray}
where $\bm{q}$ is an multi-layer perceptron (MLP) called a predictor. The output of the predictor is normalized.
The predictors are to give the losses asymmetricity, which is empirically beneficial for overall performances. 
\footnote{Previous work \cite{grill2020bootstrap} refers to the final-part MLP of the trained network to as the projector. While we follow the backbone-projector-predictor setting, in formulation, we include the projector in $\bm{f}$ for notation simplicity.}

In gradient-based optimization of the loss in Eq. \ref{eqn:loss}, updates of the intermediate layers $\bm{f}_{N-1}, \bm{f}_{N-2}, ..., \bm{f}_{1}$ are purely based on backpropagation,
which might be indirect. Here, our motivation is to expose the intermediate layers directly to their own learning objectives.

\subsection{Ladder Siamese Network}
To enhance learning of intermediate layers, we add losses on the basis of
the intermediate representations in deep nets.
We denote the intermediate representations corresponding to the two views $\bm{x}^a$ and $\bm{x}^b$ by $\bm{z}^a_1, \bm{z}^a_2, ..., \bm{z}^a_{N-1}$ and $\bm{z}^b_1, \bm{z}^b_2, ..., \bm{z}^b_{N-1}$ respectively.
Using these, the overall loss is defined by
\begin{eqnarray}
L_{\text{all}} = L(\bm{y}^a, \bm{y}^b) +\sum_{i=1}^{N-1} w_i L_i(\bm{z}^a_i, \bm{z}^b_i),
\end{eqnarray}
where $L$ denotes the final-layer loss and $L_i$ denotes the $i$-th
intermediate loss.
We introduce loss weights $w_i$ to control the balance between the losses.
Usual Siamese SSL can be seen as a special case of Ladder Siamese SSL where $w_1 = w_2 = ... = w_{N-1} = 0$.

For the concrete form of $L_i(\bm{z}^a_i, \bm{z}^b_i)$, we use an adaptation
of the BYOL loss (Eq. \ref{eqn:interloss}) for the intermediate layers,
which is defined by
\begin{eqnarray}\label{eqn:interloss}
L_{i} &=& |\bm{q}_i(\bm{y}^a_i) - \bm{y}^b_i|_2^2 ,\\ \nonumber
\bm{y}^k_i &=& \bm{p}_i(\text{avgpool}(\bm{z}^k_i)) \;\;\; (k=a, b).
\end{eqnarray}
This is near identical to Eq. \ref{eqn:loss}, except that
each level of the losses has its own projector $\bm{p}_i$ and predictor
$\bm{q}_i$, and global-pooling layers are added in a side-branching manner
apart from the main stream of the backbone network (but note that this is still equivalent to the BYOL loss that reuses the global pooling in the backbone network). Figure \ref{fig:predictors}a illustrates this intermediate-layer predictor.

A concern in this multi-loss setup is the number of hyperparameters, which increase the cost of hyperparameter searches for optimal training.
However, we empirically show that an easy heuristic can reduce hyperparameters by
\begin{eqnarray} \label{eqn:weights}
 w_i = 2^{i-N} w,
\end{eqnarray}
where $w$ is a loss weight-coefficient newly introduced instead of $w_1,  w_2,  ..., \text{and } w_{N-1}$. 
This is to simply halve the loss weight
for every one-stage shallower part of the network.

For implementation, we set the intermediate losses on res2, res3, and res4 in addition to the final stage res5 on ResNets \cite{he2016deep}. We do not set the loss on conv1, the first block that consists of a convolution and a pooling, because it seems too powerless to learn consistency against data augmentations.

\subsection{Dense loss for lower-layer supervision}
\begin{figure}[t]
    \centering
    \includegraphics[width=0.94\linewidth]{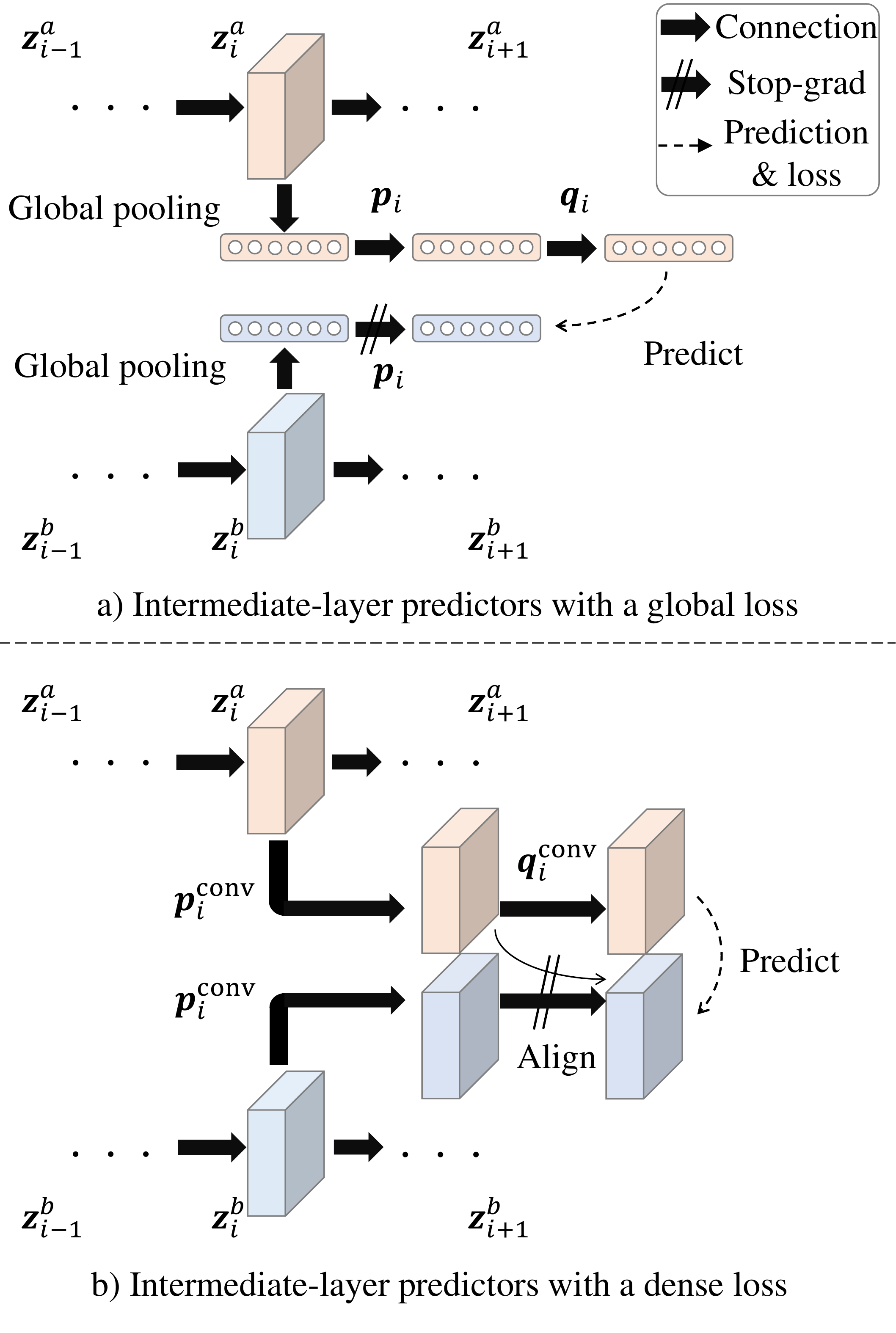}
    \vspace{-2mm}
    \caption{Structures of our intermediate predictors and losses. a) global version. b) dense version.}
    \vspace{-5mm}
    \label{fig:predictors}
\end{figure}

While a naive configuration where all-level losses are set to be the same as Eq. \ref{eqn:loss} is possible, and in a later section, we see that it is suitable for image-level classification, there is a remaining design space for varying each level loss.
We exploit this to enable the coexistence of global and local factors within single networks.

We put dense losses for lower (input-side) parts of a network, and global losses for higher (output-side)  parts.
In literature, dense losses \cite{wang2021dense,xie2021propagate,zhang2022densel} equipped with stronger locality-aware supervisory signals have advantages in object detection and
segmentation, while they degrade classification accuracies.
Here, our intention is to enhance role division of the lower layers and higher layers.
Such differentiation in hierarchical networks may naturally emerge \cite{raghu2021vision,ma2015hierarchical}, and we aim to enhance it to improve locality awareness without largely sacrificing classification accuracies.

Inspired by DenseCL \cite{wang2021dense}, we newly design the DenseBYOL loss, which is 
a non-contrastive counterpart of DenseCL designed on the basis of MoCo-style contrastive learning.
The purpose of this re-invention is to avoid potential ill effects caused by combining contrastive and non-contrastive losses, and maintain conciseness of our BYOL-based codebase.
We define our DenseBYOL loss by
\begin{eqnarray}\label{eqn:denseloss}
L_{\text{Dense}}(\bm{y}^a, \bm{y}^b) &=& |\bm{q}^\text{conv}(\bm{y}^a) - \text{align}(\bm{y}^b; \bm{y}^a)|_2^2 ,\\
\text{align}(\bm{y}^b; \bm{y}^a) &=&  [\bm{y}^b_{u,v} | (u, v) = \text{argmax}_{u, v} \langle\bm{y}^a_{i, j}, \bm{y}^b_{u, v}\rangle ]_{i, j}, \nonumber \\ 
\bm{y}^k &=& \bm{p}^{\text{conv}}(\bm{z}^k) \;\;\;\;\; (k=a, b), \nonumber
\end{eqnarray}
where $\text{\it align}$ is a spatial resampling operator that
picks up corresponding points $(u, v)$ and their feature vectors $\bm{y}^b_{u, v}$
for every $\bm{y}^a_{i, j}$ on the basis of the cosine similarity function denoted by $\langle\cdot, \cdot\rangle$.
The projector and predictor is replaced by $\bm{p}^{\text{conv}}$, $\bm{q}^{\text{conv}}$, the $1 \times 1$ convolution-based projector and predictor, which no longer require global pooling.
Figure \ref{fig:predictors}b illustrates this dense version of the predictor.
Pseudo code of Eq. \ref{eqn:denseloss} is shown in Supplementary Material.

In our Ladder Siam framework, we set the dense losses in the lower half of the network, i.e., res2 and res3 of a ResNet, and the global losses in the others (i.e., res4 and res5).
\ryoshiha{Following DenseCL \cite{wang2021dense}, we used the dense losses in combination with the global ones by averaging.}

\subsection{Do More Intermediate-layer Losses Mean More Risks of a Collapse?}
Non-contrastive Siamese learning has risks of a {\it collapse},
a phenomenon of networks falling into trivial solutions by learning
a constant function, which is useless for representation learning.
Ladder Siam adds intermediate losses, and needs to avoid the collapse of
all losses for successful training.
This might seem intuitively difficult.

However, in fact, we show that the intermediate losses do not provide
new trivial solutions in addition to the final loss.
In other words, all intermediate trivial solutions belong to the final loss's trivial solutions.
This can be formally written as follows:
\begin{theorem}
In a deep net denoted by Eq. \ref{eqn:deepnet}, 
a set of parameters $\Theta_l$ such that causes the collapse of the intermediate representation $\bm{z}_l$, is a subset of a set of parameters $\Theta_m$ such that causes the collapse of $\bm{z}_m$ when $l \le m$.
\end{theorem}
\begin{proof}[Sketch of proof]
When $l \le m$, $\bm{z}_m$ can be written using $\bm{z}_l$ and a part of the net by
\begin{eqnarray}
\bm{z}_m &=&  \bm{f}_m \circ \bm{f}_{m-1}\circ ... \circ \bm{f}_{l+1}(\bm{z}_l).
\end{eqnarray}
Given $\bm{z}_l$ collapsed,
\begin{eqnarray}
\bm{z}_l &=& \text{const.}\\ \nonumber
\Rightarrow \bm{z}_m &=&  \bm{f}_m \circ \bm{f}_{m-1}\circ ... \circ \bm{f}_{l+1}(\text{const.}) \\
&=&  \text{const.},
\end{eqnarray}
which is summarized into $\bm{z}_l = \text{const.} \Rightarrow \bm{z}_m= \text{const.} $
In the parameter space, this means
\begin{equation}
   \Theta_m = \{\bm{\theta} | \bm{z}_m = \text{const.} \} \supseteq \Theta_1 =\{\bm{\theta} | \bm{z}_l = \text{const.} \}. \nonumber
\end{equation}
\end{proof}
This means that the possible ranges of the trivial solutions have a nested structure $\Theta_1 \subseteq \Theta_{2} \subseteq \hdots \subseteq \Theta_{N}$.
Thus, we only have to avoid the collapse in $\Theta_{N}$ to avoid the collapse in all other intermediate representations, \ryoshiha{which we have already succeeded  in BYOL training.}
In experiments, we did not observe any hyperparameter setting where Ladder BYOL collapsed but BYOL did not, which agrees with this analysis.

As a limitations of this analysis, our statement is only applicable when the collapse can be complete; studies indicated that a collapse can be dimensional \cite{jing2021understanding} or partial dimensional \cite{li2022understanding}, where
the representations are not constant but strongly correlated.

%-------------------------------------------------------------------------
\section{Experiments}
\label{sec:exp}

We evaluate Ladder Siam's effectiveness as a versatile representation learner in
various vision tasks.
Following the standard protocol in prior work \cite{he2020momentum,zhang2022densel},
we first pretrained our networks with the proposed method using the ImageNet dataset,
and then finetuned them or built classifiers on them as feature extractors with frozen parameters in the downstream tasks.

\subsection{Pretraining} \vspace{-2mm}
We pretrained Ladder Siam on the ImageNet-1k \cite{deng2009imagenet} dataset (also known as ILSVRC2012)
in unsupervised fashion, i.e., without using labels.
We used 100-epoch and 200-epoch training with cosine annealing \cite{loshchilov2016sgdr} without restart as default 
because this schedule is the most widely used.
We followed BYOL \cite{grill2020bootstrap} in other settings;
as an optimizer, we used LARS \cite{you2017large} with a batch size 4,096,  initial learning rate of 7.2, and weight decay of 0.000001.

\paragraph{Method configuration} We trained two types of Ladder Siam variations. \lstinline{Ladder-BYOL} is the simpler one where all 
intermediate losses are the BYOL-style global loss described in Eq. \ref{eqn:loss}.
We followed BYOL \cite{grill2020bootstrap} in other implementation details including the application of stop-grad and momentum encoder.
\lstinline{Ladder-DenseBYOL} is the dense-loss-equipped alternative more oriented toward dense-prediction tasks e.g., segmentation; it replaced the intermediate losses on the earlier-half stages by the dense loss described in Eq. \ref{eqn:denseloss}.
For comparisons, we additionally implemented \lstinline{DenseBYOL}, which has no intermediate losses but a top-level loss of Eq. \ref{eqn:denseloss}.
As a backbone architecture, we used ResNet50 \cite{he2016deep} as default, since it is well used and the most compatible in comparison with other methods.
We set the loss weights at res2, res3, res4, and res5 to 1/16, 1/8, 1/4, and 1, respectively, as default. The impact of the setting is investigated in a later section.

\paragraph{Hardware and time consumption} We used KVM virtual machines on our private cloud infrastructure. Each machine has eight NVIDIA A100-80GB-SXM GPUs, 252 vCPUs, and 1 TB memory.
They took around 60 hours for our 200-epoch pretraining.

\begin{table}[t]
  \caption{Results of BYOL and our Ladder-BYOL.}
  \vspace{-3mm}
  \label{table:vsbyol}
  \centering
  \begin{tabular}{c|ccc}
  & BYOL & Ladder-BYOL & \\
  \hline  \hline
  {\it 100-epoch pretraining} & & \\
  IN acc@1 & 67.4 & $\bm{68.3}$ (\textcolor[RGB]{0,150,0}{+ 0.9}) \\
  CC box mAP & 39.3 & $\bm{40.5}$ (\textcolor[RGB]{0,150,0}{+ 1.2}) \\
  VOC mIoU & 63.8 & $\bm{66.6}$ (\textcolor[RGB]{0,150,0}{+ 2.8}) \\
  \hline
  {\it 200-epoch pretraining} & & \\
  IN acc@1 & 71.7 & $\bm{72.8}$ (\textcolor[RGB]{0,150,0}{+ 1.1}) \\
  CC box mAP & 40.9 & $\bm{41.4}$ (\textcolor[RGB]{0,150,0}{+ 0.5}) \\
  VOC mIoU & 64.3 & $\bm{67.4}$ (\textcolor[RGB]{0,150,0}{+ 3.1}) \\
  \hline
  {\it 400-epoch pretraining} & & \\
  IN acc@1 & 73.1 & $\bm{73.6}$ (\textcolor[RGB]{0,150,0}{+ 0.5}) \\
  \end{tabular}
  \vspace{-5mm}
\end{table}

\begin{table*}[t]
  \caption{Performance comparison in various vision tasks by state-of-the-art SSL methods and ours. \ryoshiha{{\bf Bold} indicates the best and \underline{underline} indicates the second best results.} $\dagger$: our reproduced results using released pretrained weights. *: minor differences in pretraining settings, see main texts for details.}
  \vspace{-3mm}
  \label{tab:sota}
  \centering
  % メモ：とりあえずたくさんのせる、ページ数が足りなくなったら削るかも
  \begin{tabular}{c||c|cc|cc||c}
    \hline
    & Classification  & \multicolumn{2}{c|}{Detection in COCO}  &    \multicolumn{2}{c||}{Semantic segmentation} & Avg. rank   \\
    Method & IN linear acc. & box mAP & mask mAP & VOC mIoU & Cityscapes mIoU  \\
    \hline \hline
    %BYOL \cite{grill2020bootstrap} & 71.7 & 38.4  & 34.9 & 64.3 & 71.6 \\  % omitted due to redundancy
    ReSim \cite{xiao2021region} & 66.1 & 40.0 & 36.1 & -- & \underline{76.8} \\
    DetCo \cite{xie2021detco} & 68.6 & 40.1 & 36.4 & -- & 76.5  \\
    DenseCL \cite{wang2021dense} & 63.3 & 40.3 & 36.4 & \textbf{69.4} & 75.7 & 6.4 \\
    PixPro \cite{xie2021propagate} & 66.3* & 40.5 & 36.6  & -- & 76.3 \\
    HAI-SimSiam \cite{he2019rethinking} & {70.1} & {--} & {--} & -- & -- \\
    LEWEL-BYOL \cite{huang2022learning} & \textbf{72.8} & \underline{41.3} & \textbf{37.4} & 65.7$\dagger$ & 71.3$\dagger$ & \underline{3.4} \\
    RegionCL-SimSiam* \cite{xu2021regioncl} & 71.3 & 38.8 & 35.2 & -- & -- \\
    RegionCL-DenseCL \cite{xu2021regioncl} & 68.5 & 40.4 & 36.7 & 64.8$\dagger$ & 74.1$\dagger$ & 6.2\\
    CsMl \cite{xu2022seed} &  71.6 & 40.3 & 36.6 & -- &  -- \\
    DenseSiam \cite{zhang2022densel} & -- & 40.8 & 36.8 & -- &  \textbf{77.0} \\
    \hline
    Ladder-BYOL (ours) & \textbf{72.8} & \textbf{41.4} & \underline{37.2} & 67.4 & 73.9 & \textbf{3} \\
    Ladder-DenseBYOL (ours) & 72.0 & 41.1 & 37.0 & \underline{68.6} & 75.2 & \underline{3.4} \\
    \hline
  \end{tabular}
  \vspace{-4mm}
\end{table*}

\subsection{Downstream tasks} \vspace{-2mm}
\paragraph{Classification}
We conducted linear probing using ImageNet-1k.
Linear classifiers were trained on the frozen representation using SGD.
The training of the classifier was done during 100 epochs with cosine annealing.
%following a standard setting \cite{he2020momentum}.
We used the mmselfsup \cite{mmselfsup2021} codebase.
%We regard this evaluation as a most basic one, because
%the pretraining and downstream-training images are the same,
%and there is no concern about domain generalizability of the learned representation.

\paragraph{Detection}
We finetuned Mask R-CNN \cite{he2017mask} with FPN \cite{lin2017feature} on the COCO dataset \cite{lin2014microsoft}, \lstinline{train2017} for training and \lstinline{val2017} for evaluation.
Since Mask R-CNN jointly solves box-based detection and instance segmentation,
we trained single Mask R-CNN models for both box-based and mask-based evaluation. 
The training schedule was set to {\it 1$\times$ schedule},
since longer training schedules tend to make detection performances similar regardless of initialization with pretrained models or random ones \cite{he2019rethinking}.

\paragraph{Segmentation}
We finetuned FCNs \cite{long2015fully} on PASCAL VOC \cite{everingham2010pascal} and Cityscapes \cite{cordts2016cityscapes}.
While we did not see a major consensus among the SSL literature on segmentation-evaluation protocol, we used FCN-D8, which is an FCN modified to have eight-pixel stride by dilated convolutions, provided by mmsegmentation \cite{mmseg2020} as the simplest option. 
%We re-evaluated compared SSL methods with FCN-D8 only if corresponding results were not reported in the original papers and pretrained weights were published.
In PASCAL VOC, we used \lstinline{train_aug2012} for training. We set the input resolution to $512 \times 512$ and training iterations to 20k.
In Cityscapes, we used the \lstinline{train_fine} subset for training.
We set input resolution to $769 \times 769$ and training iterations to 40k.
This setting is the same with as that of \cite{he2020momentum, zhang2022densel}.

\subsection{Results} \vspace{-2mm}
\paragraph{Comparisons with the baseline}
We first compare our Ladder-BYOL model with BYOL, which our implementation is based on and we regard as a baseline. The results are shown in Table \ref{table:vsbyol}.
We observed improvements over the baselines with our Ladder version
on all datasets and training schedules we used.
The relative improvements were 0.9 \% points in ImageNet linear classification (IN),
1.2 \% points in COCO (CC) detection, and 2.8 \% points in VOC segmentation when we adopted 100-epoch pretraining. 
With 200-epoch pretraining, 
the improvements were 1.0 \% points in IN,
0.5 \% points in CC detection, and 3.1 \% points in VOC segmentation,
which shows that our Ladder Siam training framework is consistently beneficial in combination with BYOL.
The improvement is + 0.5 \% points in IN with the longer 400-epoch pretraining.
This relative improvement is a bit smaller than in shorter-term training, and we regard this as the result of faster convergence.

\paragraph{Comparisons with state of the art methods}
We show the results in Table \ref{tab:sota}.
We selected recent ResNet50-based SSL methods that do not rely on extra region extractors or multi-crop strategies, for which we can draw a fair comparison.
The reported scores of the compared methods are from the original papers and the DenseSiam \cite{zhang2022densel} paper, which paid great effort for fair 200-epoch-pretraining-based comparisons based on their reimplementations, unless otherwise noted.
In the same way, we reported 200-epoch results of our models.
Incidentally, as marked by * in the table, we placed the classification accuracy of the 400-epoch-pretrained model for PixPro due to the unavailability of 200-epoch weights.
We placed classification accuracy of the 100-epoch model for RegionCL-SimSiam, which
we expect to be similar to its 200-epoch results due to the fast convergence and saturation of SimSiam-based methods \cite{chen2020simple}.

Given the diversity of the downstream tasks, we do not see a single clear winner. 
However, our Ladder-BYOL maintain a balance of downstream performances at a high level by being the best in ImageNet-1k (IN) linear classification, the best in COCO (CC) box-based detection, and the second best in instance segmentation.
For example, LEWEL-BYOL \cite{huang2022learning} performed well and  similarly in classification and detection to ours, but was found to be less generalizable to segmentation.
\ryoshiha{In contrast, DenseCL \cite{wang2021dense} was the best in VOC segmentation but at the cost of classification accuracy.}
Our Ladder-DenseBYOL is the second best in VOC semantic segmentation, while it has similar but slightly worse performances in the other tasks than Ladder-BYOL. Thus, it can be regarded as a still versatile but somewhat segmentation-oriented backbone.

\paragraph{Component-wise comparisons}
We further show component-wise comparisons that focus on methods that have connections on the underlying ideas and consist of similar components to ours.
First, we compare the effect of adding intermediate losses
with DetCo \cite{xie2021detco}.
Table \ref{table:mls} shows the ImageNet classification-performance changes by adding MLS by the intermediate losses, which were provided by the original paper \cite{xie2021detco} as a part of an ablative study and computed by us. While the two DetCo variants degraded
their classification performances by MLS, ours improved in contrast.
A possible cause of this reversal is the difference of contrastive and non-contrastive losses; the contrastive 
loss used in DetCo can be bottlenecked by its reliance on negative pairs, which are sometimes too hard to distinguish from positives \cite{li2022understanding}. This might be more harmful when the losses are assigned to less powerful intermediate layers.

Table \ref{table:dense} compares the effect of incorporating dense SSL losses \cite{wang2021dense,xie2021propagate} in various base SSL methods. 
Note that PixPro and DenseCL used dense losses as their top-level supervision, but our Ladder models exploited dense losses on intermediate layers.
Regardless of base methods or dense loss types,
the addition of the dense losses degraded classification and improved segmentation in all examine conditions.
However, classification degradation in our Ladder-DenseBYOL, which is - 0.7 \%-points, is softer than in the others.
This observation suggests that dense losses as the intermediate supervision is a reasonable way to relieve classification degradation.

\begin{table}[t]
  \caption{Effect of adding multi-level supervision on ImageNet accuracy.}
  \vspace{-3mm}
  \label{table:mls}
  \centering
  \begin{tabular}{c|cc}
  & Baseline & + Multi-level sup.  \\
  \hline
  DetCo w/o GLS \cite{xie2021detco} & 64.3 & 63.2 (\color{red}{- 1.1})  \\
  DetCo w/ GLS \cite{xie2021detco} & 67.1 & 66.6  (\color{red}{- 0.5}) \\
  Ours & 71.7 & 72.8 (\textcolor[RGB]{0,150,0}{+ 1.1}) \\
  \end{tabular}
  \vspace{-3mm}
\end{table}

\begin{table}[t]
  \caption{Effect of adding dense SSL losses.}
  \vspace{-3mm}
  \label{table:dense}
  \centering
  \begin{tabular}{c|c|cc}
  & Dense & IN acc@1 & VOC  mIoU  \\
  \hline \hline
  {\it base} MoCov2 \cite{chen2020improved} & & 67.6 & 67.5 \\
  DenseCL \cite{wang2021dense} & \checkmark & {63.3} (\color{red}{- 4.3})  &{69.4} (\textcolor[RGB]{0,150,0}{+ 1.9}) \\
  \hline
  {\it base} BYOL \cite{grill2020bootstrap} & & 67.4 & 63.8 \\
  PixPro \cite{xie2021propagate} & \checkmark & 66.3 (\color{red}{- 1.1})  & {65.0}  (\textcolor[RGB]{0,150,0}{+ 1.2}) \\
  DenseBYOL & \checkmark & 65.2 (\color{red}{- 2.2})  & {65.2}  (\textcolor[RGB]{0,150,0}{+ 1.4}) \\
  \hline 
  {\it base} Ladder-BYOL & & 72.8 & {67.4} \\
  Ladder-DenseBYOL & \checkmark & 72.0 (\color{red}{- 0.8})  & {68.6}  (\textcolor[RGB]{0,150,0}{+ 1.2}) \\
  \end{tabular}
  \vspace{-7mm}
\end{table}

\begin{table}[t]
  \caption{Effect of replacing global losses with dense losses. ``G'' and ``D'' denote the usage of global and dense losses respectively. We used 100-epoch pretraining.}
  \vspace{-4mm}
  \label{table:densemls}
  \centering
  \begin{tabular}{c|cccc|cc}
  &  &  &  &  & IN & VOC   \\
   & res2 & res3 & res4 & res5 &  acc@1 &  mIoU  \\
  \hline
  BYOL & -- & -- & -- & G &  67.4  & 64.3\\
  DenseBYOL & -- & -- & -- & D & 65.2 & 65.1 \\
  Ladder-B & G & G & G & G & $\bm{68.8}$ & 66.6 \\
  Ladder-DB & D & D & G & G & 68.2 & $\bm{67.1}$ \\
   & D & D & D & D & 67.7 & 66.0 \\
  \end{tabular}
  \vspace{-3mm}
\end{table}

\begin{table}[t]
  \caption{Effect of the loss-weight hyperparameters.}
  \vspace{-3mm}
  \label{table:hparam}
  \centering
  \begin{tabular}{c|c|c|cc}
  Method & Epochs & Loss weights & Acc@1 & mAP \\
  \hline
  Ladder- & 100 & [0, 0, 0, 1] & 67.4 & 39.3 \\
  BYOL &  & [1/32, 1/16, 1/8, 1] & $\bm{68.8}$ & 40.2  \\
    &  & [1/16, 1/8, 1/4, 1] & 68.3 & 40.5\\
    &  & [1/8, 1/4, 1/2, 1] & 66.9 & $\bm{40.7}$ \\
    &  & [1/8, 1/8, 1/8, 1] & 68.5 & 40.2 \\
    \hline
  Ladder- & 200 & [0, 0, 0, 1] & 71.7 & 40.7\\
  BYOL &  & [1/32, 1/16, 1/8, 1] & 71.7 & 41.1\\
    &  & [1/16, 1/8, 1/4, 1] & $\bm{72.8}$ & $\bm{41.4}$ \\
    &  & [1/8, 1/4, 1/2, 1] & 72.5 & 40.9\\
  \end{tabular}
  \vspace{-3mm}
\end{table}

\begin{table}[t]
  \caption{Results of classification using intermediate layers.}
  \vspace{-3mm}
  \label{table:intercls}
  \centering
  \begin{tabular}{c|cccc}
  Method & res2 & res3 & res4 & res5 \\
  \hline
  BYOL & 28.5 & 40.9 & 56.8 & 68.1 \\
  Ladder-BYOL & {\bf 31.9} & {\bf 47.2} & {\bf 61.7} & {\bf 68.7} \\
  \end{tabular}
  \vspace{-3mm}
\end{table}

\begin{figure}[t]
    \centering
    \includegraphics[width=1.0\linewidth]{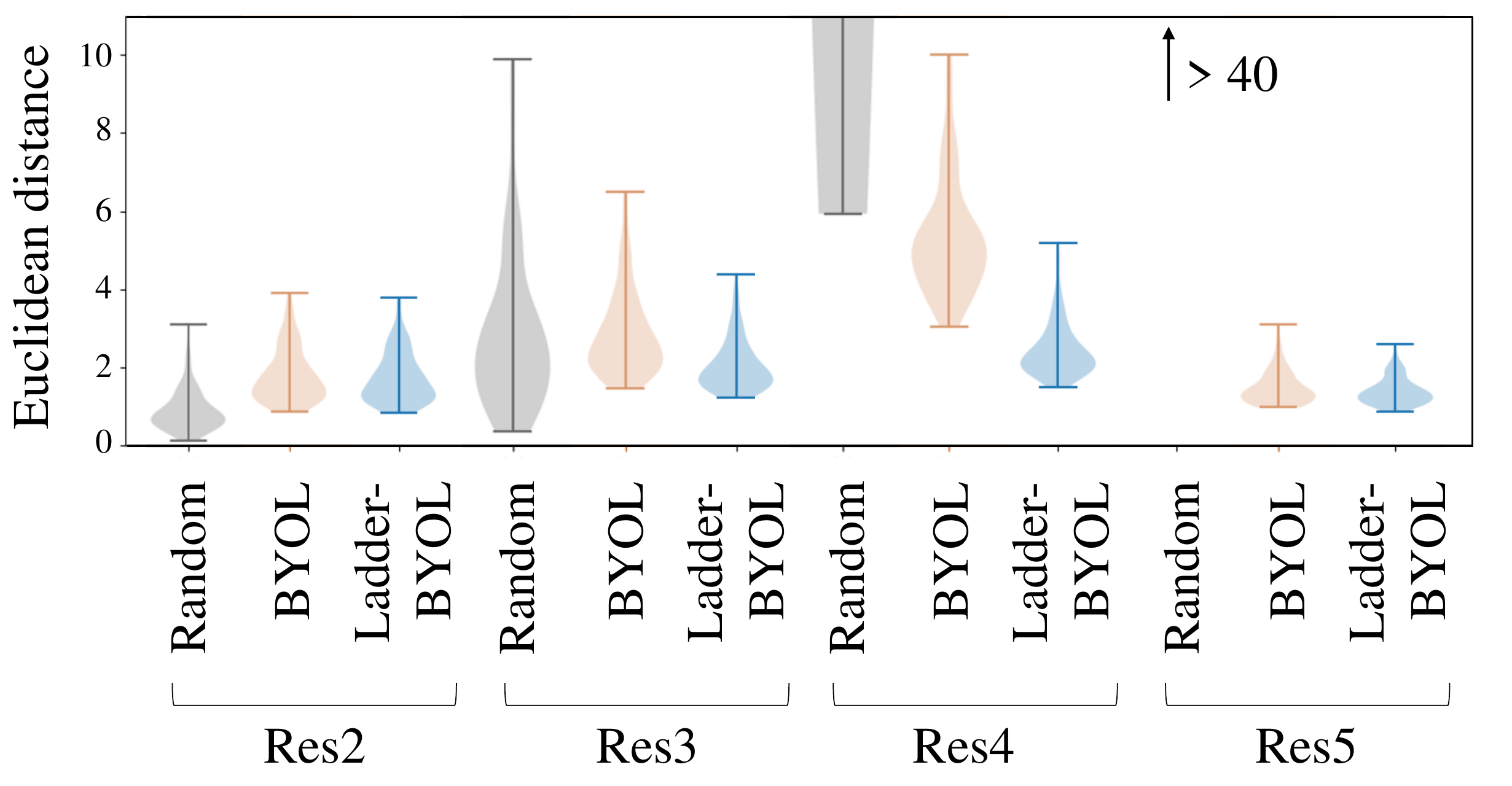}\vspace{-4mm}
    \caption{Distribution of euclidean distance between the two data-augmented views measured in each stage of the network.}\vspace{-5mm}
    \label{fig:violins}
\end{figure}

%\begin{comment}
%\vspace{-8mm}
\begin{figure*}[ht]
    \begin{tabular}{cc}
    \hspace{-2mm}
      \begin{minipage}[t]{0.27\hsize}
        \centering
        %\vspace{-3mm}
        \includegraphics[width=1\linewidth]{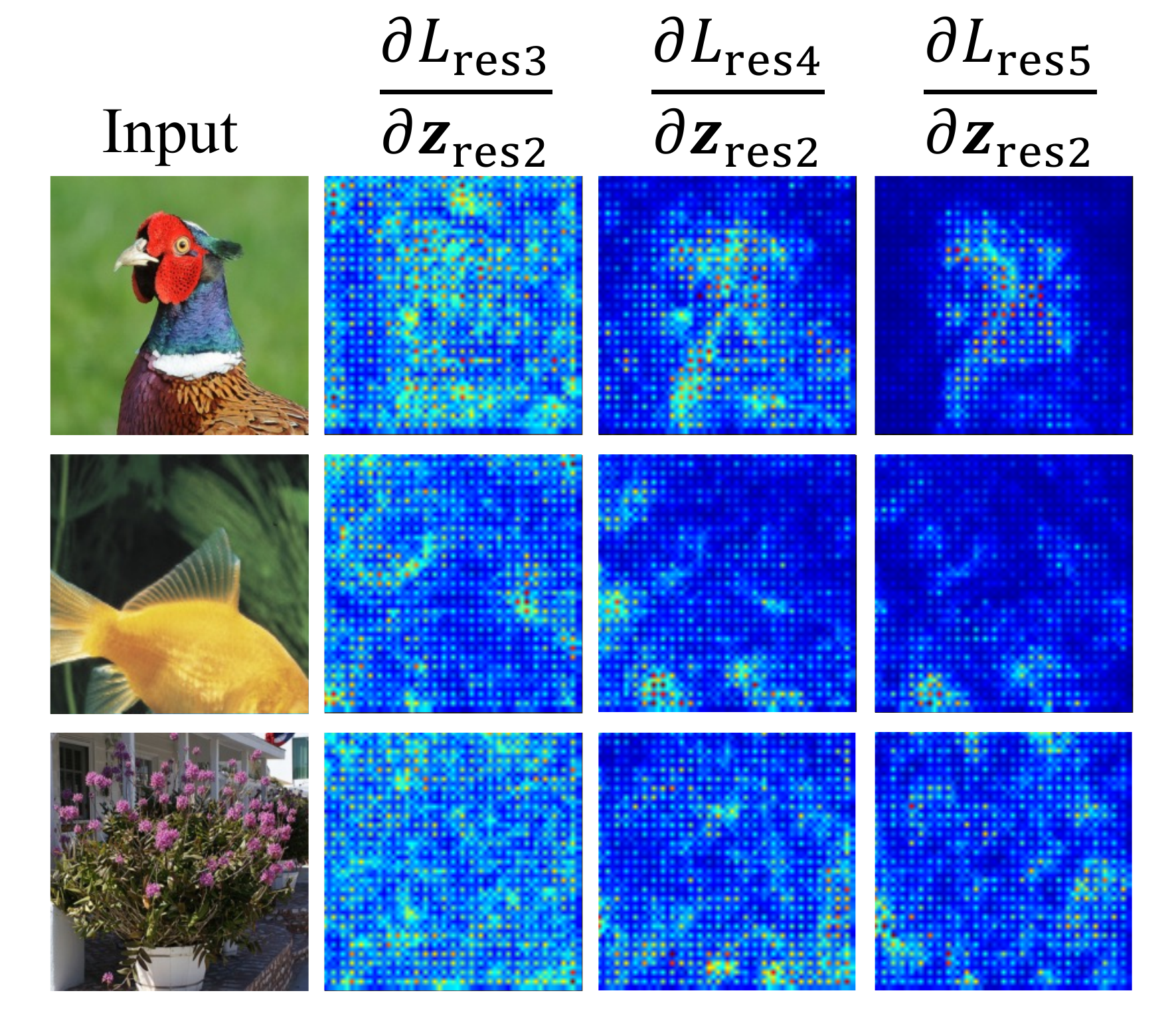}
        \vspace{-8mm}
        \caption{Visualization of gradients, i.e., supervisory signals during training at the earliest-stage provided by each intermediate loss.}
        \label{fig:grad}
        \end{minipage} 
        \hspace{3mm}
        \begin{minipage}[t]{0.67\hsize}
        \centering
        \includegraphics[width=1.05\linewidth]{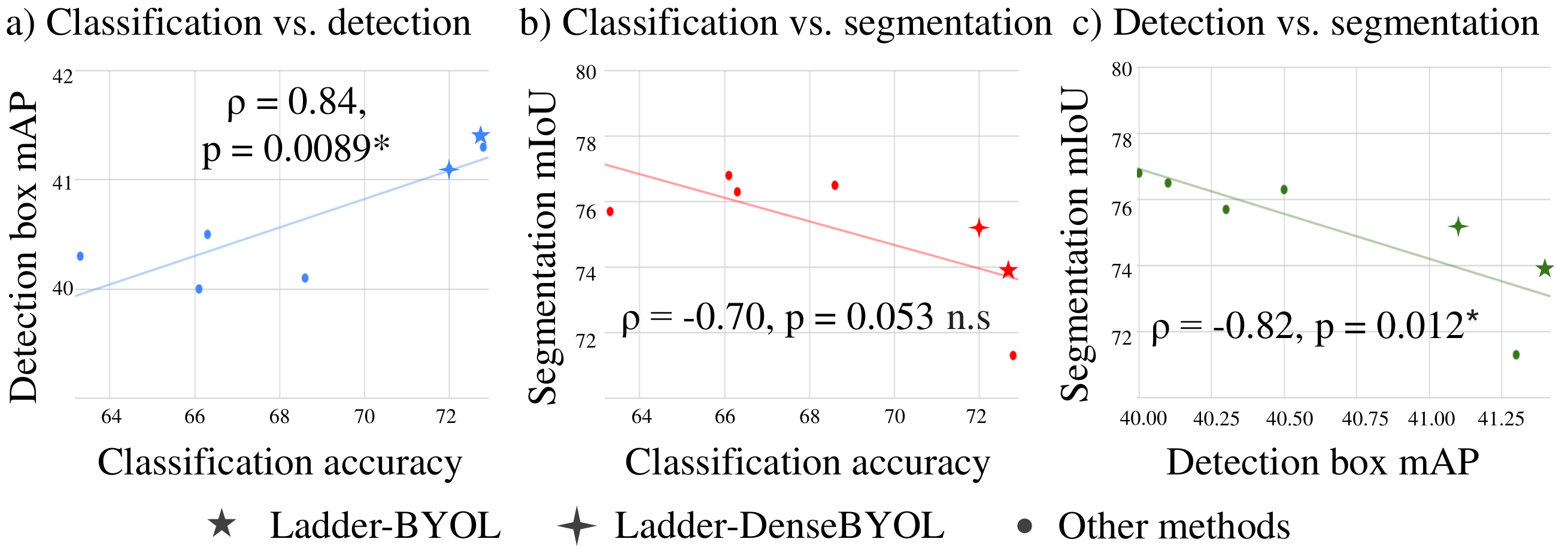}
        \caption{Correlations between the downstream performances. Classification-vs.-detection's is positive while segmentation-vs.-the-other's is negative.}
        \label{fig:meta}
        \end{minipage} 
    \end{tabular}
     %\caption{minipagetest}
     \vspace{-6mm}
  \end{figure*}
%\end{comment}

\paragraph{Hyperparameters and ablations}
We conducted ablation analyses of the intermediate losses and the results are summarized Table \ref{table:densemls}.
Ladder-BYOL and Ladder-DenseBYOL were confirmed to outperform their single-loss counterparts.
We additionally tested a Ladder-DenseBYOL variation where all losses are dense, but it was suboptimal both in classification and segmentation, offering more evidence of the effectiveness of mixing dense and global losses.

Next, we investigated the impact of the intermediate-loss weights as hyperparameters. The results are summarized in Table \ref{table:hparam}.
In setting the loss weights, we followed Eq. \ref{eqn:weights}
and modified the coefficient $w$ to control the strength of the overall intermediate losses.
An interesting trend was seen in the 100-epoch pretraining; larger intermediate loss weights degraded classification and improved detection. This implies that models could be tuned to be classification-oriented or detection-oriented just by the hyperparameter. 
However, the same trend was not observed in the 200-epoch pretraining.
This might be related to the stronger convergence by longer training, which results in a single good setting rather than selectable variations.

\begin{comment}
\paragraph{Collapse observation}
To investigate the models' behavior when a collapse occurs we searched for the optimizer hyperparameters that cause a collapse in training.
Here, we used a smaller-dataset setting of the CIFAR-10 dataset and Wide ResNet34, which enabled more try-and-errors and detailed analyses.
As a result, we found a case that causes a collapse with the optimizer AdamW and an initial learning rate of 0.2.
Figure \ref{fig:collapse} shows the training curve and transition of downstream accuracy.
In this setting, BYOL and Ladder-BYOL both caused a collapse,
which was observed by the sudden degradation of downstream classification accuracy. However, Ladder-BYOL was 
durable for longer training iterations than BYOL, supporting improved training stability by Ladder Siam.
\end{comment}

\paragraph{Analyses on intermediate representations}
\ryoshiha{
We investigated how intermediate representations differed by direct exposure to supervisions in Ladder models.
Table \ref{table:intercls} summarizes linear probing results of intermediate representations in each stage.
We used IN classification here, and improvements in all intermediate layers were confirmed.
Figure \ref{fig:violins} shows distributions of euclidean distance between two random data-augmented views measured in each level as violin plots. The distance was computed using representation vectors after global average pooling.
Ladder training provided stronger consistency against the data augmentation to the intermediate layers lower than res4, which is seemingly the source of the improved intermediate-layer discriminability.
}

\vspace{-4mm}\paragraph{Gradient visualization}
After confirming the effectiveness of Ladder Siam, we further investigated whether the roles of each intermediate loss as a supervisor are similar or whether some sort of role divisions emerge.
We observed signs of role divisions in Fig. \ref{fig:grad}, which visualized gradients of each-level loss with reference to an intermediate representation.
Given the representation $\bm{z}_{\text{res}2}$ and the losses $L_{\text{res}3}$, $L_{res4}$ and $L_{res5}$ viewed as a function on $\bm{z}_{\text{res}2}$, we computed $\frac{\partial L_{\text{res-}i}}{\partial \bm{z}_{\text{res}2}}$,
which is the contribution of each $ L_{\text{res-}i}$ to the total gradient of $\bm{z}_{\text{res}2}$ for $i = 3, 4, 5$.
We excluded $L_{\text{res}2}$ in the visualization because the global average pooling on  $\bm{z}_{\text{res}2}$ provides spatially uniform gradients, which is improper for visualization. 
For plotting, we took the absolute-sum along the channel axis and visualized 2-D patterns of gradient magnitude.
In Fig. \ref{fig:grad}, the later-level gradients are more focused on objects while the earlier-level ones are more globally distributed on both foregrounds and backgrounds, which can be useful to widely collect learnable factors.
At the same time, the earlier losses might be non-object-centric
when disrupted by background clutters, and here we hypothesize that later-level
and earlier-level losses work complementarily together.

\paragraph{Meta-analyses: how do downstream-task performances correlate?}
Finally, we are interested in the correlation patterns seen among the various SSL methods' downstream-task performances summarized in Table \ref{fig:grad}. 
Is a good performance in one downstream task a sign of a good performance in another?
To answer this question, we conducted correlation coefficient testing over the scores in Table \ref{fig:grad} as a post-hoc analysis.
Here, we calculated Pearson's correlation coefficients of a) ImageNet linear accuracy vs. COCO detection box mAP, b) ImageNet linear accuracy vs. Cityscapes segmentation mIoU, and c) COCO detection box mAP vs. Cityscapes segmentation mIoU.
For segmentation, we selected Cityscapes as there are more available data points (i.e., reported scores).
Due to the triple comparison, we applied Bonferroni's correction \cite{dunn1961multiple,bretz2016multiple} in statistical testing.
As a result, a positive correlation with $\rho = 0.84$ was observed in the classification-detection comparison, and a negative correlation
in the classification-segmentation ($\rho = -0.70$) and detection-segmentation ($\rho = -0.82$) as shown in Fig. \ref{fig:meta}.
While a positive correlation is amenable to SSL's dogma to pursue generally reusable representations, the negative ones between Cityscapes segmentation and the others are notable in future research designs.
  
%-------------------------------------------------------------------------
\vspace{-4mm}\section{Conclusion}\vspace{-2mm}
In this paper, we presented Ladder Siamese Network, 
conceptually simple yet effective framework to stably learn
versatile self-supervised representations.
Other than effectiveness, Ladder Siam's advantage is its flexibility
to incorporate various learning mechanisms in each level,
which may inspire more sophisticated designs of self-supervised learning objectives.
In future, we will explore further effective combinations of loss functions
such as region-proposal-based and unsupervised-segmentation-based ones 
with our Ladder Siam framework.

\paragraph{Limitations}
\ryoshiha{While Ladder Siam worked well with hierarchical representations of conv nets, its applicability to Vision Transformers \cite{dosovitskiy2020image} remains an open question.
Hierarchical Transformers \cite{wang2021pyramid,liu2021swin,heo2021rethinking} are promising in vision tasks, and they would be compatible with MLS. 
However, non-hierarchical Transformer was found to be competitive \cite{li2022exploring}.
The question of {\it whether we should apply MLS to Transformers} interacts with {\it whether Transformers should be hierarchical}, and they might need parallel consideration.}

\label{sec:conclusion}

\section*{Acknowledgements}\vspace{-2mm}
The authors would like to thank members of Tech Lab and Image-processing Group, Yahoo Japan Corporation for the helpful comments and discussion.

%%%%%%%%% REFERENCES
{\small
\bibliographystyle{ieee_fullname}
\bibliography{egbib}
}

\end{document}